%% file: supermodularity_ssl.tex
\documentclass{article}
\usepackage{spconf,amsmath,graphicx}
\usepackage{url}
\usepackage{graphicx}
\usepackage{epstopdf}
\usepackage{algorithm}
\usepackage{algorithmicx}
\usepackage{algpseudocode}
\usepackage{cite}
\usepackage{color}
\usepackage{amsmath,amssymb,amsfonts,amsthm}
\usepackage{bm}
\usepackage{setspace}
\usepackage{multirow}
\usepackage{mathtools}
\usepackage{subcaption}

\newtheorem{theorem}{Theorem}
\newtheorem{lemma}[theorem]{Lemma}
\newtheorem{cor}[theorem]{Corollary}
\def\bOmega{\boldsymbol{\Omega}}
\def\zeros{\boldsymbol{0}}
\def\ones{\boldsymbol{1}}
\def\bn{{\mathbf{n}}}
\DeclareMathOperator{\tr}{tr}
\DeclareMathOperator{\rank}{rank}

\include{notation_20171024}

\title{On the Supermodularity of Active Graph-based Semi-supervised Learning with Stieltjes matrix Regularization}
\name{Pin-Yu Chen$^{*}$\thanks{$^*$P.-Y. Chen and D. Wei contributed equally to this work.} and Dennis Wei$^{*}$}
\address{IBM Research AI, Thomas J. Watson Research Center, Yorktown Heights, NY 10598, USA\\
pin-yu.chen@ibm.com,  dwei@us.ibm.com
	}
\begin{document}
\ninept
\maketitle
\begin{abstract}
Active graph-based semi-supervised learning (AG-SSL) aims to select a small set of labeled examples and utilize their graph-based relation to other unlabeled examples to aid in machine learning tasks. It is also closely related to the sampling theory in graph signal processing. In this paper, we revisit the original formulation of graph-based SSL  and prove the supermodularity of an AG-SSL objective function under a broad class of regularization functions parameterized by Stieltjes matrices. Under this setting, supermodularity yields a novel greedy label sampling algorithm with guaranteed performance relative to the optimal sampling set. Compared to three state-of-the-art graph signal sampling and recovery methods on two real-life community detection datasets, the proposed AG-SSL method attains superior classification accuracy given limited sample budgets.
\end{abstract}
\section{Introduction}
\label{sec:intro}
Given a set of labeled examples and a set of unlabeled examples, the goal of semi-supervised learning (SSL) is to leverage the inherent relations between  labeled and unlabeled examples for improved learning performance. In particular, graphs are often used to specify the relation between examples \cite{liu2012robust,Sandryhaila13,Sandryhaila14,Shuman13}, where a node represents an example and an edge reveals the similarity between two examples, quantified by the edge weight. Furthermore, in \emph{active} graph-based SSL (AG-SSL), the set of labeled examples is not given in advance but rather is chosen strategically 
based on the graph representation.  Active learning is appealing in cases when label querying is expensive while unlabeled examples are readily available.

In recent years, AG-SSL has received great attention in the context of sampling and recovery methods for graph signal processing \cite{gadde2014active,anis2017sampling}, where labels are viewed as signals atop the underlying graph. This body of work has focused on the notions of graph Fourier transform and bandlimitedness, in analogy with the sampling theory in traditional (i.e.\ temporal) signal processing. For a graph signal $\bx \in \bbR^n$, its graph Fourier transform (GFT) is defined as $\bxt=\bF \by$, where $\bxt$ is a vector of the  corresponding graph Fourier coefficients. The Fourier transform matrix $\bF \in \bbR^{n \times n}$ is derived from the eigenvectors of 
a real symmetric matrix $\bX$ that represents the graph. 
Popular choices of $\bX$ are $\bX=\bA$, the (weighted) adjacency matrix, $\bX=\bL$, the (unnormalized) graph Laplacian matrix defined as $\bL=\bD-\bA$, with $\bD$ denoting the diagonal matrix of node degrees, and $\bX=\bLNN$, the normalized 
 graph Laplacian matrix defined as $\bLNN=\bD^{-1/2} \bL \bD^{-1/2}$. For a given $\bF$, the graph signal $\bx$ is said to be $K$-bandlimited if there are only $K$ nonzero coefficients in $\bxt$, and $K$ is referred to as the bandwidth. This notion of bandlimitedness, i.e.\ restriction to a subspace, has been essential to 
 the recent development of graph signal sampling and recovery schemes \cite{anis2014towards,chen2015discrete}.  
 
Different from the perspective of graph signal processing,  in this paper we revisit the original formulation of graph-based SSL proposed in \cite{belkin2004regularization,belkin2006manifold}. In particular, we dispense with any assumptions on the bandlimitedness of graph signals. We formulate the problems of inferring a graph signal (labels) based on an incomplete set of noisy linear observations as well as selecting the set of observations to maximize the precision of this inference. We prove that under a broad class of regularization functions parameterized by the family of Stieltjes matrices, the objective function is supermodular in the set $\cS \subset \cV$ of labeled examples. A real symmetric matrix $\bX$ is said to be a (possibly singular) Stieltjes matrix if it is positive semidefinite and its off-diagonal entries are non-positive, which  applies to the two popular matrices $\bL$ and $\bLNN$ used in graph signal processing. Under this setting, we propose 
an efficient greedy AG-SSL algorithm for selecting a set $\cS$ of samples with $|\cS|=s$. Supermodularity then guarantees that the loss resulting from 
the proposed AG-SSL method is bounded by a constant factor relative to the optimal sampling set $\cS^*$, where finding $\cS^*$ requires searching over all $\binom{n}{s}$ possible combinations of examples and is hence computationally infeasible. In contrast, under a different bandlimited signal model, supermodularity does not hold and only weak supermodularity can be ensured \cite{chamon2016near,chamon2017greedy}, resulting in a worse performance guarantee.
Tested on two real-life community detection datasets and given limited sample budgets, the proposed AG-SSL method attains superior classification accuracy than three state-of-the-art graph signal sampling and recovery methods.

\section{Problem formulation}
\label{sec:prob}

Let $G=(\cV,\cE)$ be an undirected weighted connected graph, where $\cV$ and $\cE$ are the sets of nodes and edges with size $|\cV|=n$ and $|\cE|=\ell$, respectively. The edge weights in $G$ are represented by a symmetric (weighted) adjacency matrix $\bA \in \bbR^{n \times n}$, where $\bA_{ij} > 0$ specifies the weight of an edge $(i,j) \in \cE$, and $\bA_{ij} = 0$ otherwise. The (weighted) degree matrix $\bD$ is then defined as a diagonal matrix with the vector $\bA \ones$ along the diagonal, where $\ones$ is the all-ones vector, and the graph Laplacian matrices $\bL$ and $\bLNN$ are defined as $\bL = \bD - \bA$ and $\bLNN = \bD^{-1/2} \bL \bD^{-1/2}$. The notation $\tr$ denotes matrix trace and $\|\bx\|$ refers to the Euclidean norm of a vector $\bx$.

We denote by $\bx$ the signal (i.e., labels) to be recovered on the graph $G$. In this work, $\bx$ is modelled as a random signal with a multivariate Gaussian prior distribution,
\begin{equation}\label{eqn:prior}
p(\bx) \propto \exp\left( -\frac{1}{2} \bx^T \bOmega_0 \bx \right),
\end{equation}
where the mean is taken to be zero without loss of generality. The precision matrix $\bOmega_0$ is symmetric positive semidefinite with rank either $n-1$ or $n$.  The case $\rank(\bOmega_0) = n-1$ corresponds to an improper Gaussian distribution.  This allows $\bOmega_0$ to be chosen to be proportional to the unnormalized or normalized graph Laplacian matrices, which have one-dimensional nullspaces (consisting of multiples of $\ones$ in the case of $\bL$).

A total of $m$ noisy linear observations of the form $\by = \bC \bx + \bn$ are available to be taken, where $\bC$ is a fixed $m\times n$ measurement matrix, $\bn \sim \cN(\zeros, \sigma^2\bI)$ represents i.i.d.\ Gaussian noise,\footnote{While our results can also be shown to hold for the noiseless case $\bn = \zeros$, we do not present that case here.} and $\bI$ is the identity matrix.  We are restricted however to choosing a subset $\cS$ of size $s < m$, resulting in the partial observations $\by_{\cS} = \bC_{\cS} \bx + \bn_{\cS}$, where $\bC_{\cS}$ is the submatrix of $\bC$ with rows indexed by $\cS$.  The posterior distribution of $\bx$ given $\by_{\cS}$ is given by 
\begin{equation}\label{eqn:posterior}
p(\bx \mid \by_\cS) \propto \exp\left[ -\frac{1}{2} \left( \frac{1}{\sigma^2} \left\lVert \by_\cS - \bC_\cS \bx \right\rVert^2 + \bx^T \bOmega_0 \bx \right) \right],
\end{equation}
which is also Gaussian with precision matrix
\begin{align}
\label{eqn_precision}
\bOmega(\cS) = \bOmega_0 + \frac{1}{\sigma^2} \bC_{\cS}^T \bC_{\cS}.
\end{align}
We assume that $\bOmega(\cS)$ is non-singular for $\cS \neq \{\varnothing\}$.  In other words, if $\rank(\bOmega_0) = n-1$, the addition of any rank-$1$ term $\bC_v^T \bC_v$, $v \in \cV$, results in a full-rank matrix.  This assumption is satisfied for example if $\bOmega_0$ is proportional to the graph Laplacian matrix $\bL$ and the rows of $\bC$ satisfy $\bC_v \ones \neq 0$ for all $v \in \cV$. 

Given the posterior distribution \eqref{eqn:posterior} and the invertibility assumption above, the estimator $\widehat{\bx}$ that minimizes the mean squared error (MSE) with respect to $\bx$ is the posterior mean, 
\begin{align}
\label{eqn_predictor}
\widehat{\bx} = \mathbb{E}[\bx \mid \by_\cS] = \frac{1}{\sigma^2} \bSigma(\cS) \bC_{\cS}^T \by_{\cS},
\end{align}
where $\bSigma(\cS) = \bOmega(\cS)^{-1}$ is the posterior covariance matrix.  The corresponding minimum MSE is 
\[
\mathbb{E}\left[ \left\lVert \widehat{\bx} - \bx \right\rVert^2 \mid \by_\cS \right] = \tr \cov(\bx \mid \by_\cS) = \tr \bSigma(\cS).
\]
We thus define the objective function 
\begin{equation}\label{eqn:objective}
f(\cS) = \tr \bSigma(\cS) = \tr\bOmega(\cS)^{-1} = \tr \left( \bOmega_0 + \frac{1}{\sigma^2} \bC_{\cS}^{T} \bC_{\cS} \right)^{-1},
\end{equation}
taking $f(\cS) = +\infty$ if $\bOmega(\cS)$ is singular.  The goal of AG-SSL is to select a set $\cS \subset \cV$ of size $\lvert\cS\rvert = s$ to minimize $f(\cS)$.

We note that maximizing the posterior distribution \eqref{eqn:posterior} is equivalent to solving the following minimization problem with $\alpha=\sigma^2$:
\begin{align}
\label{eqn_GSSL}
\textnormal{minimize}_{\bx} \left\lVert \by_\cS - \bC_\cS \bx \right\rVert^2 + \alpha \bx^T \bOmega_0 \bx,
\end{align} 
whose solution is also given by \eqref{eqn_predictor}. The formulation (\ref{eqn_GSSL}) coincides with the formulation of graph-based SSL proposed in \cite{belkin2004regularization,belkin2006manifold}.

\section{Main Results}

\subsection{Properties of objective function}
\label{sec:properties}

In this section, we study properties of the objective function $f(\cS)$ defined in \eqref{eqn:objective}.  These properties all relate to the change in $f(\cS)$ when a node is added to an existing subset.  We denote by $\delta_v(\cS) = f(\cS) - f(\cS\cup\{v\})$ the decrease due to adding node $v$ to $\cS$.  
\begin{lemma}\label{lem:decrease}
	Assume that $\bOmega(\cS)$ in \eqref{eqn_precision} is invertible.  Then the decrease in $f$ due to adding node $v$ to $\cS$ is given by 
	\begin{equation}\label{eqn:decrease}
	\delta_v(\cS) = f(\cS) - f(\cS\cup\{v\}) = \frac{\left\lVert \bSigma(\cS) \bC_v^T \right\rVert^2}{\sigma^2 + \bC_v \bSigma(\cS) \bC_v^T}.
	\end{equation}
\end{lemma}
\begin{proof}
	We have 
	\[
	\bOmega(\cS\cup\{v\}) = \bOmega_0 + \frac{1}{\sigma^2} \left(\bC_{\cS}^{T} \bC_{\cS} + \bC_{v}^{T} \bC_{v}\right) = \bOmega(\cS) + \frac{1}{\sigma^2} \bC_{v}^{T} \bC_{v}
	\]
	so adding a node results in a rank-1 update to $\bOmega(\cS)$.  By the matrix inversion lemma, 
	\begin{equation}\label{eqn:MTinv}
	\bSigma(\cS\cup\{v\}) = \bSigma(\cS) - \frac{\bSigma(\cS) \bC_v^T \bC_v \bSigma(\cS)}{\sigma^2 + \bC_v \bSigma(\cS) \bC_v^T}.
	\end{equation}
	The result is obtained by taking the trace and applying the identity $\tr(\bX \bY) = \tr(\bY\bX)$ to the numerator.
\end{proof}
\noindent If $\bOmega(\cS)$ is singular, then $\delta_v(\cS) = +\infty$ since the addition of $v$ is assumed to make $\bOmega(\cS\cup\{v\})$ invertible and $f(\cS\cup\{v\})$ is hence finite. 
\begin{cor}
	The objective function $f(\cS)$ is a monotonically non-increasing set function.
\end{cor}
\begin{proof}
	If $\bOmega(\cS)$ is invertible, then $\bSigma(\cS)$ is positive definite and the decrease \eqref{eqn:decrease} from adding a node is always non-negative.  If $\bOmega(\cS)$ is singular then $\delta_v(\cS) = +\infty$.
\end{proof}

Next we consider conditions under which $f(\cS)$ is a supermodular set function. 
For $f$ to be supermodular, we must have $\delta_v(\cS) \geq \delta_v(\cT)$ for any $\cS$, $\cT = \cS \cup \{u\}$, $u \notin \cS$, and $v \notin \cT$.  For general $\bC$ and $\bOmega_0$, it is possible to construct counterexamples where supermodularity does not hold. We omit such a construction here due to limited space.

In the remainder of this section, we specialize to the case $\bC = \bI$ and assume that  $\bOmega_0$ is a (possibly singular) Stieltjes matrix, i.e., a symmetric positive (semi)definite matrix with non-positive off-diagonal entries.  The latter assumption is satisfied if $\bOmega_0\propto\bL$ or $\bOmega_0\propto\bLNN$.  With $\bC = \bI$ and $\cS \neq \{\varnothing\}$, $\bOmega(\cS) = \bOmega_0 + \frac{1}{\sigma^2} \bI_{\cS}^T \bI_{\cS}$ remains Stieltjes and is invertible (i.e.\ positive definite).  We may then exploit the \emph{inverse-positivity} property of Stieltjes matrices \cite{berman1994nonnegative}, namely that $\bOmega(\cS)^{-1} = \bSigma(\cS) \geq 0$ element-wise.  Under the above assumptions, the function $f(\cS)$ is supermodular as shown below.
\begin{theorem}
	\label{thm_supermodular}
	The objective function $f(\cS)$ is supermodular if $\bC = \bI$ and $\bOmega_0$ is a (possibly singular) Stieltjes matrix.
\end{theorem}
\begin{proof}
	As discussed earlier, it suffices to show that $\delta_v(\cS) \geq \delta_v(\cT)$ for any $\cT = \cS \cup \{u\}$, $u \notin \cS$, and $v \notin \cT$.  If $\bOmega(\cS)$ is singular, then $\delta_v(\cS) = \infty$ and the inequality holds.  Otherwise, we use Lemma~\ref{lem:decrease} and set $\bC = \bI$ to obtain 
	\begin{align}
	\delta_v(\cS) - \delta_v(\cT) &= \frac{\left\lVert \bSigma(\cS)_v^T \right\rVert^2}{\sigma^2 + \bSigma(\cS)_{vv}} - \frac{\left\lVert \bSigma(\cT)_v^T \right\rVert^2}{\sigma^2 + \bSigma(\cT)_{vv}}\nonumber\\
	&\propto \left( \sigma^2 + \bSigma(\cT)_{vv} \right) \left\lVert \bSigma(\cS)_v^T \right\rVert^2\nonumber\\
	&\qquad {} - \left( \sigma^2 + \bSigma(\cS)_{vv} \right) \left\lVert \bSigma(\cT)_v^T \right\rVert^2.\label{eqn:decreaseDiff}
	\end{align}
	Using \eqref{eqn:MTinv} and $\bC = \bI$, $\bSigma(\cT)$ can be expressed in terms of $\bSigma(\cS)$.  Taking the $v$th row of the resulting expression yields  
	\begin{equation}\label{eqn:MTinvCol}
	\bSigma(\cT)_v = \bSigma(\cS)_v - \frac{\bSigma(\cS)_{uv}}{\sigma^2 + \bSigma(\cS)_{uu}} \bSigma(\cS)_u,
	\end{equation}
	whence 
	\begin{equation}\label{eqn:MTinvElem}
	\bSigma(\cT)_{vv} = \bSigma(\cS)_{vv} - \frac{\bSigma(\cS)_{uv}^2}{\sigma^2 + \bSigma(\cS)_{uu}}.
	\end{equation}
	We substitute \eqref{eqn:MTinvCol} and \eqref{eqn:MTinvElem} into the right-hand side of \eqref{eqn:decreaseDiff} and cancel a pair of terms.  Defining 
	\begin{equation}\label{eqn:Delta}
	\bDelta_v = \frac{\bSigma(\cS)_{uv}}{\sigma^2 + \bSigma(\cS)_{uu}} \bSigma(\cS)_u,
	\end{equation}
	the result is 
	\begin{align}
	\delta_v(\cS) - \delta_v(\cT)
	&\propto \left(\sigma^2 + \bSigma(\cS)_{vv}\right) \bDelta_v \left( 2\bSigma(\cS)_v - \bDelta_v \right)^T\nonumber\\
	&\qquad {} - \frac{\bSigma(\cS)_{uv}^2}{\sigma^2 + \bSigma(\cS)_{uu}} \left\lVert \bSigma(\cS)_v^T \right\rVert^2\nonumber\\
	&= \left(\sigma^2 + \bSigma(\cS)_{vv}\right) \bDelta_v \bSigma(\cT)_v^T
+ \bSigma(\cS)_v \times\nonumber\\
	&\quad \left[ \left(\sigma^2 + \bSigma(\cS)_{vv}\right) \bDelta_v 
- \frac{\bSigma(\cS)_{uv}^2}{\sigma^2 + \bSigma(\cS)_{uu}} \bSigma(\cS)_v \right]^T,\label{eqn:decreaseDiff2}
	\end{align}
	where in the last line we have used $\bSigma(\cT)_v = \bSigma(\cS)_v - \bDelta_v$ from \eqref{eqn:MTinvCol} and \eqref{eqn:Delta}.
	
	We show that the right-hand side of \eqref{eqn:decreaseDiff2} is non-negative.  Since $\bOmega_0$ is Stieltjes, so too are $\bOmega(\cS)$ and $\bOmega(\cT)$.  It follows that $\bSigma(\cS)_v$, $\bSigma(\cT)_v$, and $\bDelta_v$ are all element-wise non-negative (the last can be seen from \eqref{eqn:Delta}). 
	The remaining quantity in square brackets can be rewritten as 
	\begin{equation}\label{eqn:decreaseDiff3}
	\left(\sigma^2 + \bSigma(\cS)_{vv}\right) \bSigma(\cS)_u - \bSigma(\cS)_{uv} \bSigma(\cS)_v
	\end{equation}
	after pulling out a factor of $\bSigma(\cS)_{uv} \left(\sigma^2 + \bSigma(\cS)_{uu}\right)^{-1}$, which is non-negative.  To show that \eqref{eqn:decreaseDiff3} is also element-wise non-negative, we recall that $\cT = \cS \cup \{u\}$ and switch the roles of $u$ and $v$ in \eqref{eqn:MTinvCol} to derive 
	\begin{align*}
	&\bSigma(\cS\cup\{v\})_u = \bSigma(\cS)_u - \frac{\bSigma(\cS)_{uv}}{\sigma^2 + \bSigma(\cS)_{vv}} \bSigma(\cS)_v \geq 0,
	\end{align*}
	where the non-negativity is due to $\bOmega(\cS\cup\{v\})$ being Stieltjes.  It is seen that \eqref{eqn:decreaseDiff3} is a positive multiple of $\bSigma(\cS\cup\{v\})_u$.  We conclude that the right-hand side of \eqref{eqn:decreaseDiff2} is indeed non-negative as the sum of inner products of non-negative vectors.
\end{proof}

\subsection{Proposed Greedy AG-SSL Algorithm}
Based on the supermodularity result stated in Theorem \ref{thm_supermodular}, we propose a greedy  AG-SSL algorithm as summarized in Algoritm \ref{algo_AGSSL}.
\begin{algorithm}[h]
	\caption{Greedy AG-SSL Algorithm}
	\label{algo_AGSSL}
	\begin{algorithmic}
		\State \textbf{Input:} Graph $G$, regularization $\bOmega_0$, $\sigma$, \# of samples $s$
		\State \textbf{Output:} sampling set $\cS$ and predictor $\widehat{\bx}(\cS)$
		\State Initialization: $\cS=\{ \varnothing \}$
		\For{$k=1$ to $s$} 
		\State Find $v^*=\arg \min_{v \in \cV / \cS} f(\cS \cup v)$ using (\ref{eqn:objective})
		\State $\cS \leftarrow \cS \cup \{ v^* \}$
		\EndFor
		\State $\widehat{\bx}(\cS)=\sigma^{-2} \bSigma(\cS) \bI_{\cS}^T \by_{\cS}$ using (\ref{eqn_precision}) and (\ref{eqn_predictor})
	\end{algorithmic} 
\end{algorithm} 
Due to supermodularity, the proposed algorithm is guaranteed to have at most a constant-factor 
performance loss relative to the optimal sampling set, as stated in the following corollary \cite{Fujishige90}.
\begin{cor}
	Let $\cS$ be the sampling set obtained from Algorithm \ref{algo_AGSSL}, $\cS^*$ a minimizer of $f(\cS)$ in (\ref{eqn:objective}) over sets of size $|\cS|=s$, and $\cS^*_1$ a minimizer of $f(\cS)$ over singletons (e.g.\ from first iteration of Algorithm~\ref{algo_AGSSL}). Then 
    \[
    \frac{f(\cS) - f(\cS^*)}{f(\cS^*_1) - f(\cS^*)} \leq \left(1-\frac{1}{s}\right)^{s-1} \to \frac{1}{e} \textrm{ as } s \to \infty.
    \]
\end{cor}

\section{Experiments}

 \begin{table*}[t]
 	\centering
 	\caption{Fraction of overlapping samples between different AG-SSL methods (\%).}
 	\label{Table_overlap}
 	\begin{tabular}{c|cccc|ccccc}
 		\hline
 		\multicolumn{5}{c|}{Karate, \# samples $s=6$, $\bOmega_0=\bL$} & \multicolumn{5}{c}{Dolphin, \# samples $s=5$, $\bOmega_0=\bL$}   \\ \hline
 		& Proposed    & GSP      & GSO     & CRM     & \multicolumn{1}{c|}{}         & Proposed & GSP & GSO & CRM \\ \hline
 		Proposed    & 100         & 66.67    & 16.67   & 16.67   & \multicolumn{1}{c|}{Proposed} & 100      & 40  & 20  & 0   \\
 		GSP         & 66.67       & 100      & 33.33   & 0       & \multicolumn{1}{c|}{GSP}      & 40       & 100 & 0   & 0   \\
 		GSO         & 16.67       & 33.33    & 100     & 0       & \multicolumn{1}{c|}{GSO}      & 0        & 0   & 100 & 0   \\
 		CRM         & 16.67       & 0        & 0       & 100     & \multicolumn{1}{c|}{CRM}      & 20       & 0   & 0   & 100 \\ \hline
 	\end{tabular}
 \end{table*}

\subsection{Community Detection Datasets}
\label{subsec_community}
Community detection is a fundamental task in network analysis \cite{Fortunato10,CPY14ICASSP,CPY14deep}. Based on the connectivity structure of a graph, it aims to partition the nodes into well-connected groups, also known as communities. Community detection can be cast as a semi-supervised learning problem when the labels (i.e., community membership) of a subset of nodes are made available. In this section, we use the Karate club network \cite{Zachary77Karate} and the dolphin social network \cite{Lusseau03Dolphin} for performance evaluation. We purposely choose these two small networks having $n=34$ and $n=62$ nodes, respectively, due to the following reasons: (i) there are only two community labels ($+1$ and $-1$) and hence it is essentially a binary classification task, which avoids 
variations caused by different representations of multi-class labels. 
(ii) the graphs are given and therefore one avoids 
the effect of different graph construction methods on the final outcome. (iii) For active node sampling with fixed size $|\cS|=s$, the total number of sampling sets,  $\binom{n}{s}$, grows approximately as $n^s$. Therefore, smaller networks allow better tracking of the similarity of different sampling strategies in terms of the proportion of overlapping nodes.

\subsection{Comparative Methods}
\label{subsec_compare}
We compare the proposed greedy AG-SSL method in Algorithm \ref{algo_AGSSL} with the following AG-SSL methods:\\
\noindent
\textbf{Random sampling (Rand):} Randomly sample $s$ nodes and report the averaged results over 10 trials.
\\
\noindent
\textbf{Graph spectral proxy (GSP) \cite{anis2016efficient}:}  The graph spectral proxy using the regularization matrix $\bOmega_0$ for GFT  is used for node sampling (Algorithm 1 in \cite{anis2016efficient}) and the corresponding predictor is implemented. We set the parameters of GSP as $k=5$ and $r=s$.
\\
\noindent
\textbf{Graph shift operator (GSO) \cite{chen2015discrete}:} The graph shift operator $\bD^{-1} \bA$ is used for node sampling  (Algorithm 1 in \cite{chen2015discrete}) and the graph total variation minimization method \cite{SihengChen15signalrecovery} is used for prediction.  We set the parameters of GSO as $K=20$.
\\
\noindent
\textbf{Chamon-Ribeiro's method (CRM) \cite{chamon2017greedy}:} The greedy node sampling method using $\bOmega_0$ for GFT (Algorithm 2 in \cite{chamon2017greedy}). Since the CRM optimal predictor requires assumptions 
of band-limitedness and stationarity, 
we instead substitute its selection of sampled nodes into our predictor in (\ref{eqn_predictor}).
 The noise parameter of CRM is $\lambda_w=0.01$.

  \begin{figure}[t]
 	\centering
 	\begin{subfigure}[b]{.9\linewidth}
 		\includegraphics[width=\textwidth]{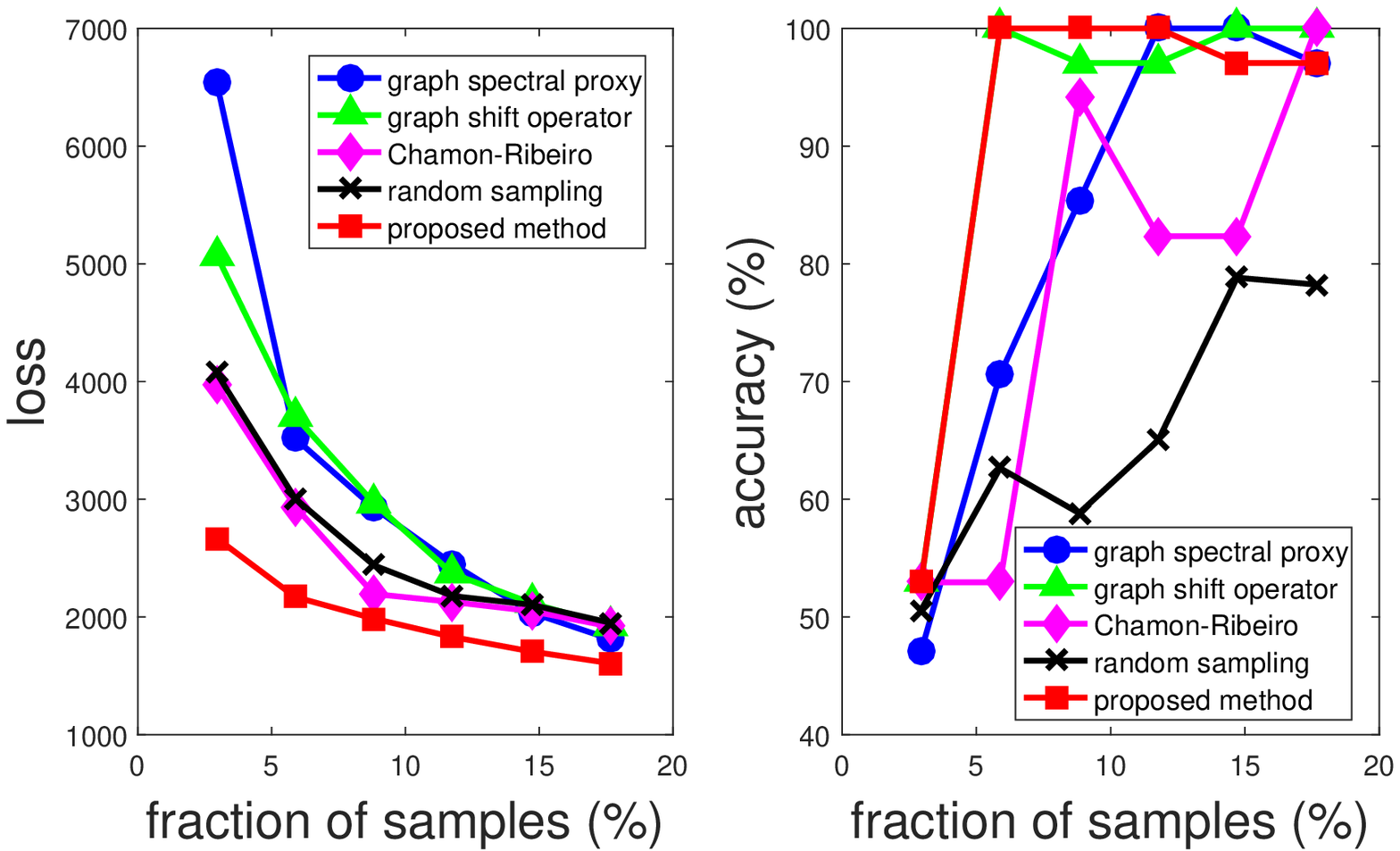}
 		\caption{Karate network.}
 	\end{subfigure}%
    \\
 	\centering
 	\begin{subfigure}[b]{.9\linewidth}
 		\includegraphics[width=\textwidth]{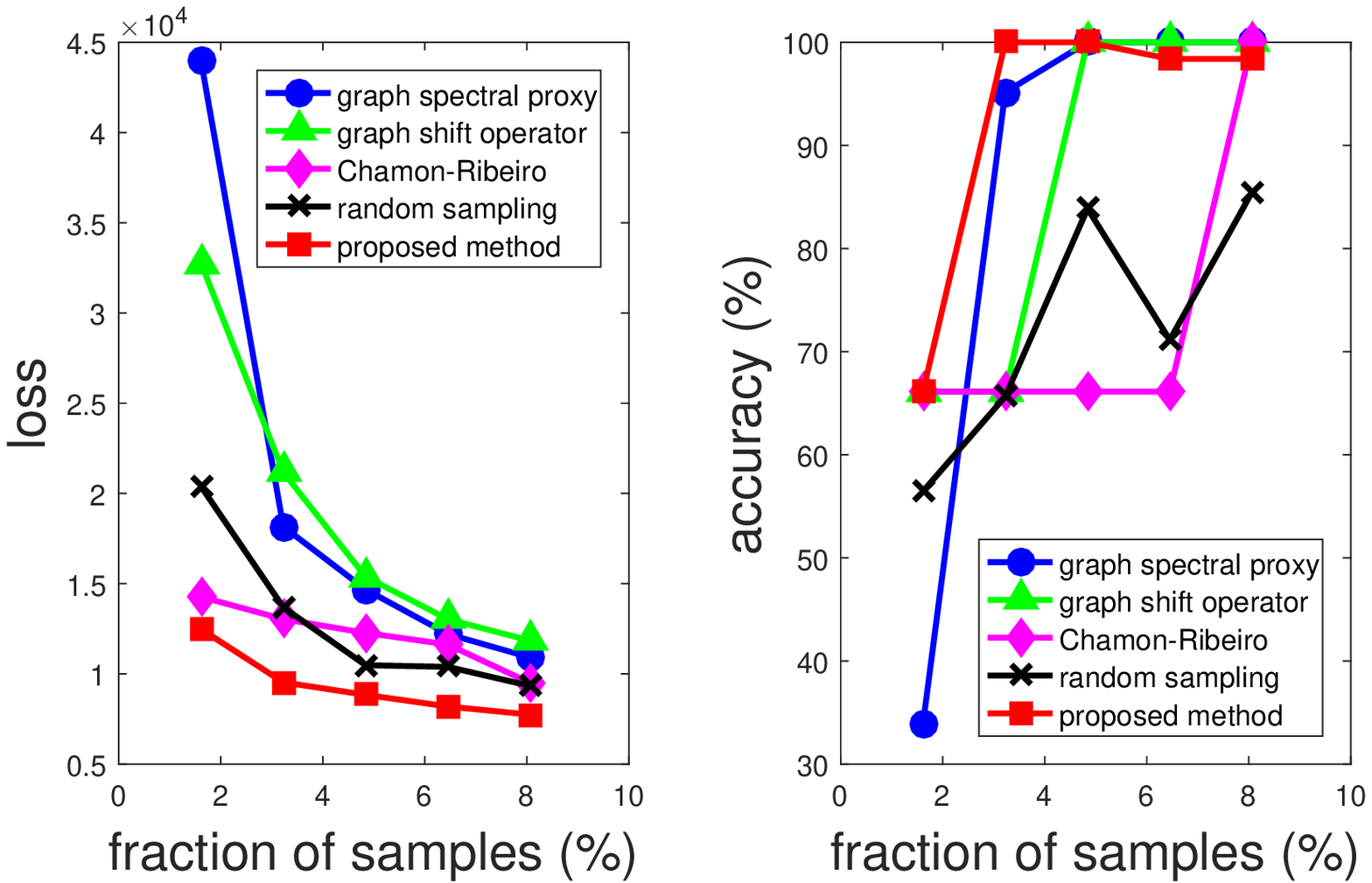}
 		\caption{Dolphin network.}
 	\end{subfigure}
 	\caption{AG-SSL using $\bOmega_0=\bL$ and $\sigma^2=1/\tr(\bOmega_0)$. The proposed method yields perfect community detection by only sampling 2 nodes in each dataset.}
 	\label{Fig_L}
 \end{figure}

 \begin{figure}[t]
 	\centering
 	\begin{subfigure}[b]{.9\linewidth}
 		\includegraphics[width=\textwidth]{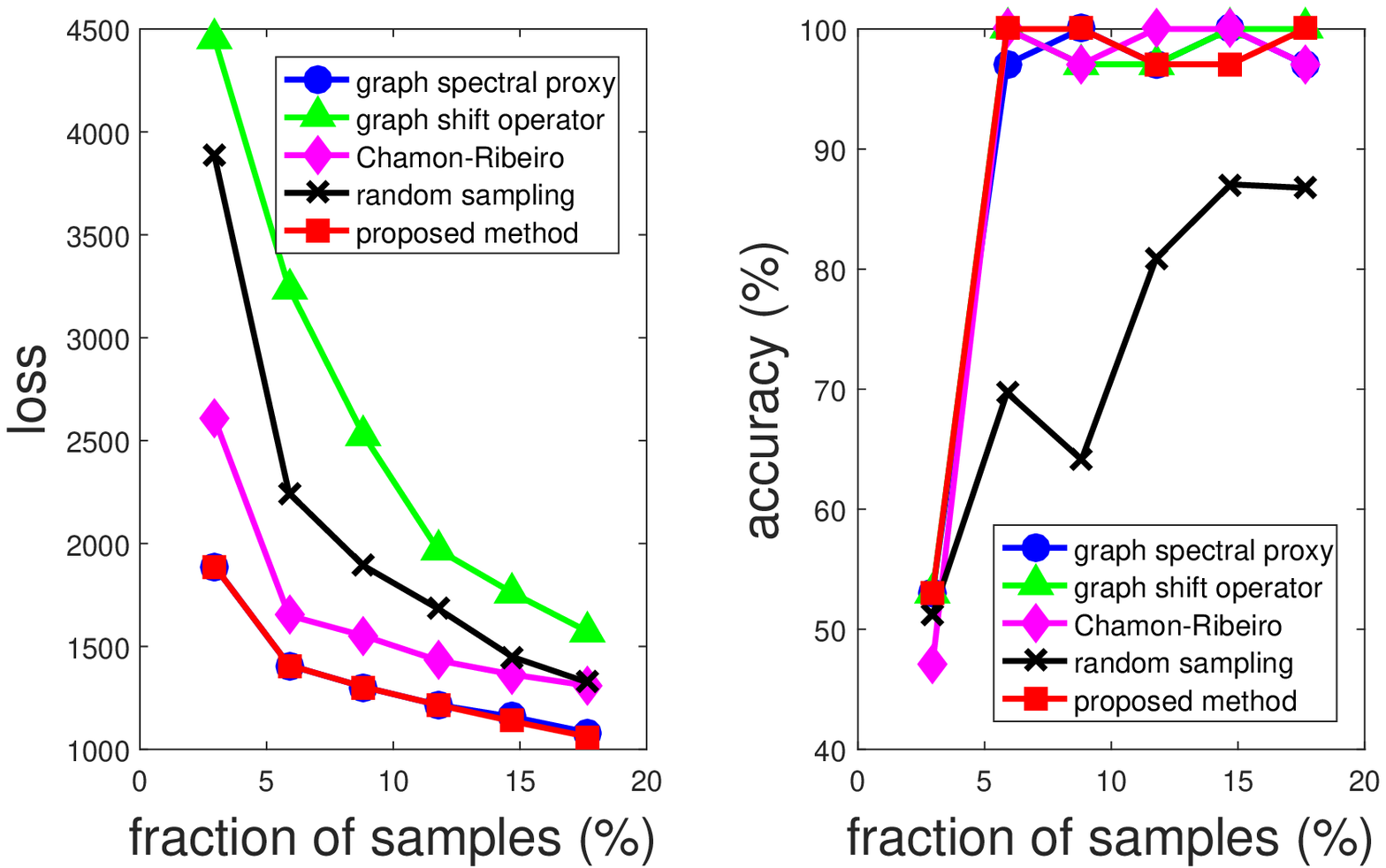}
 		\caption{Karate network.}
 	\end{subfigure}%
    \\
 	\centering
 	\begin{subfigure}[b]{.9\linewidth}
 		\includegraphics[width=\textwidth]{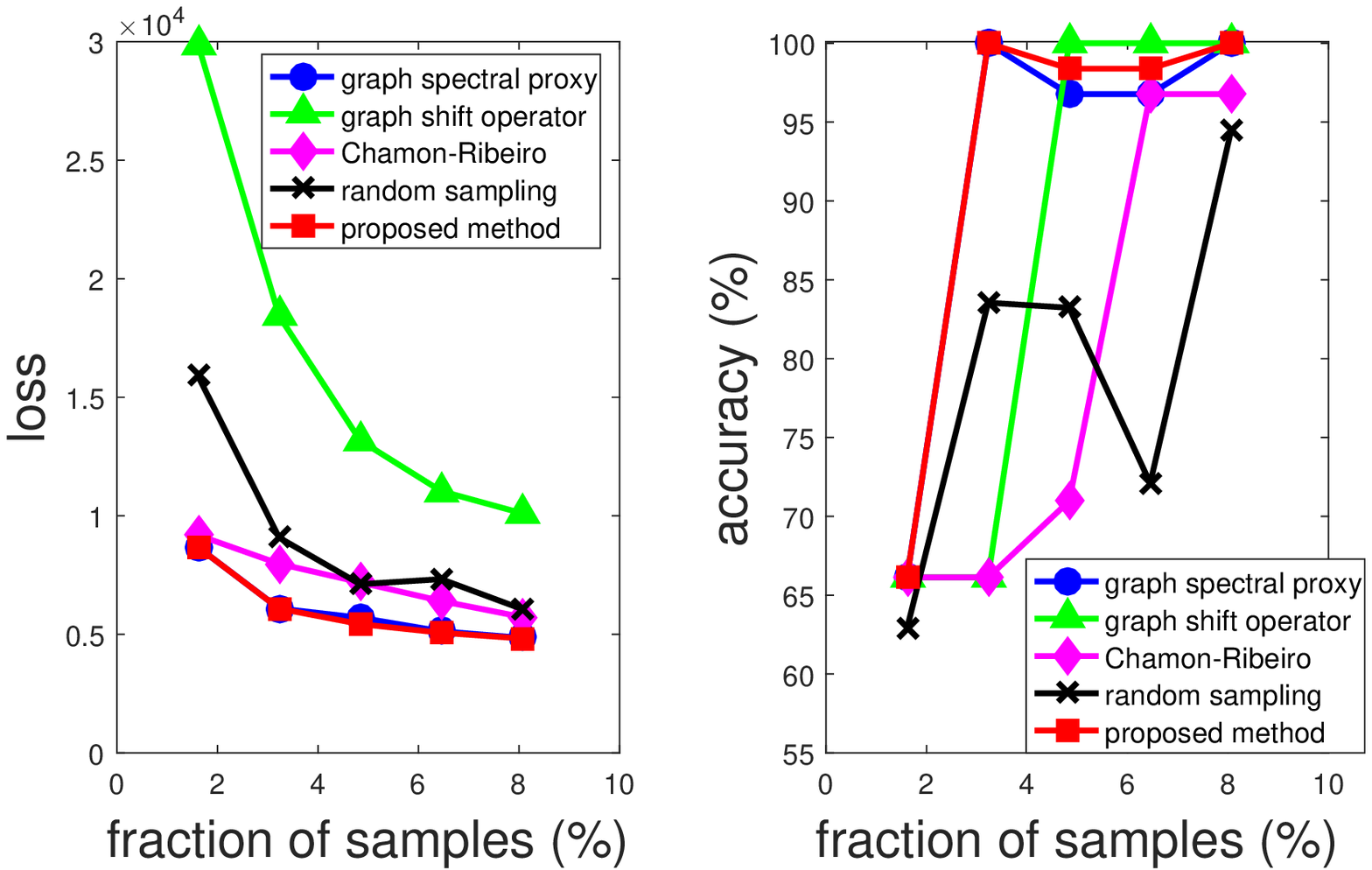}
 		\caption{Dolphin network.}
 	\end{subfigure}
 	\caption{AG-SSL  using $\bOmega_0=\bLNN$ and $\sigma^2=1/\tr(\bOmega_0)$.  The proposed method attains similar performance as in Figure \ref{Fig_L}.
 		The use of $\bLNN$ improves the accuracy of GSP and CRM.}
 	\label{Fig_LN}
 \end{figure}

\subsection{Performance Evaluation}
Since AG-SSL on the two community detection datasets discussed in Section \ref{subsec_community} can be viewed as a semi-supervised binary classification problem,  for each method we take the sign of the prediction as the final predicted label. 
For all methods other than Algorithm~\ref{algo_AGSSL}, we assume that community labels are acquired without noise. For Algorithm~\ref{algo_AGSSL}, to fully comply with the signal model in Section \ref{sec:prob}, we added an artificial zero-mean Gaussian noise with variance $\sigma^2=1 / \tr(\bOmega_0)$ to the observed labels. The parameter $\sigma^2$ can also be viewed as the regularization parameter for the Stieltjes matrix $\bOmega_0$ in (\ref{eqn_GSSL}).
Figures 1 and 2 show the loss function value from (\ref{eqn:objective}) and the prediction accuracy of different AG-SSL methods using $\bOmega_0=\bL$ and $\bOmega_0=\bLNN$, respectively, where in each curve the markers correspond to the number $s$ of selected samples. It is observed that the proposed method (Algorithm \ref{algo_AGSSL}) attains the least loss, and Rand may have lower loss than some methods. This is not surprising because the proposed method is designed to minimize the loss via greedy node sampling, while other methods have different objective functions. Nonetheless, for each method except Rand, lower loss is aligned with higher accuracy. The slight fluctuation in the accuracy curve is an artifact of the sign operation for label prediction.

 It is worth mentioning that the proposed method can achieve perfect community detection (100\% accuracy) by only sampling 2 nodes in each dataset, either using $\bL$ or $\bLNN$ as the regularizer. On the other hand, the other methods may require more samples to achieve the same performance. 
We also observe that using  $\bLNN$ instead of $\bL$ can improve the accuracy of GSP and CRM, which can be explained by the fact that selecting a different $\bOmega_0$ is equivalent to changing the Graph Fourier basis for sampling. To further study different AG-SSL methods, Table \ref{Table_overlap} displays the fraction of overlapping samples between different sampling strategies when $\bOmega_0=\bL$. An interesting finding is that the proposed method has a good amount of overlap with GSP, whereas the samples of GSO and CRM can be very different. Similar relations hold when 
$\bOmega_0=\bLNN$. We believe this distinction can be explained by the use of spectral proxy for approximating the actual bandwidth of a graph signal in GSP \cite{anis2016efficient}.

\section{Conclusion}
This paper proved the supermodularity of graph-based SSL under the family of Stieltjes matrix regularization functions, which includes the unnormalized and normalized graph Laplacian matrices that are widely used in graph signal processing. Our analysis yields an efficient greedy AG-SSL algorithm with guaranteed performance relative to the optimum. Evaluated on two community detection datasets, the proposed algorithm outperforms three recent graph signal sampling methods given limited sample budgets.
Our future work includes extension to vector-valued labels and joint optimization of the sampling set and the Stieltjes matrix for AG-SSL.

\clearpage
\bibliographystyle{IEEEbib}
\bibliography{IEEEabrv,CPY_ref_20171023}

\end{document}

%% file: notation_20171024.tex
\newcommand{\cov}{\textnormal{cov}}

\newcommand{\bx}{\mathbf{x}}
\newcommand{\by}{\mathbf{y}}

\newcommand{\bxt}{\widetilde{\mathbf{x}}}

\newcommand{\bX}{\mathbf{X}}
\newcommand{\bY}{\mathbf{Y}}

\newcommand{\bI}{\mathbf{I}}

\newcommand{\bSigma}{\mathbf{\Sigma}}

\newcommand{\bLNN}{\mathbf{L}_{N}}

\newcommand{\bDelta}{\mathbf{\Delta}}
\newcommand{\bA}{\mathbf{A}}

\newcommand{\bF}{\mathbf{F}}

\newcommand{\bD}{\mathbf{D}}
\newcommand{\bL}{\mathbf{L}}

\newcommand{\bC}{\mathbf{C}}

\newcommand{\cS}{\mathcal{S}}

\newcommand{\cN}{\mathcal{N}}

\newcommand{\cV}{\mathcal{V}}
\newcommand{\cE}{\mathcal{E}}
\newcommand{\cT}{\mathcal{T}}

\newcommand{\bbR}{\mathbb{R}}

%% file: supermodularity_ssl.bbl
\begin{thebibliography}{10}

\bibitem{liu2012robust}
Wei Liu, Jun Wang, and Shih-Fu Chang,
\newblock ``Robust and scalable graph-based semisupervised learning,''
\newblock {\em Proceedings of the IEEE}, vol. 100, no. 9, pp. 2624--2638, 2012.

\bibitem{Sandryhaila13}
A~Sandryhaila and J.M.F. Moura,
\newblock ``Discrete signal processing on graphs,''
\newblock {\em {IEEE} Trans. Signal Process.}, vol. 61, no. 7, pp. 1644--1656,
  Apr. 2013.

\bibitem{Sandryhaila14}
Aliaksei Sandryhaila and Jose~MF Moura,
\newblock ``Big data analysis with signal processing on graphs: Representation
  and processing of massive data sets with irregular structure,''
\newblock {\em {IEEE} Signal Process. Mag.}, vol. 31, no. 5, pp. 80--90, 2014.

\bibitem{Shuman13}
D.I. Shuman, S.K. Narang, P.~Frossard, A.~Ortega, and P.~Vandergheynst,
\newblock ``The emerging field of signal processing on graphs: Extending
  high-dimensional data analysis to networks and other irregular domains,''
\newblock {\em {IEEE} Signal Process. Mag.}, vol. 30, no. 3, pp. 83--98, 2013.

\bibitem{gadde2014active}
Akshay Gadde, Aamir Anis, and Antonio Ortega,
\newblock ``Active semi-supervised learning using sampling theory for graph
  signals,''
\newblock in {\em ACM International Conference on Knowledge Discovery and Data
  Mining (KDD)}, 2014, pp. 492--501.

\bibitem{anis2017sampling}
Aamir Anis, Aly~El Gamal, Salman Avestimehr, and Antonio Ortega,
\newblock ``A sampling theory perspective of graph-based semi-supervised
  learning,''
\newblock {\em arXiv preprint arXiv:1705.09518}, 2017.

\bibitem{anis2014towards}
Anas Anis, Akshay Gadde, and Antonio Ortega,
\newblock ``Towards a sampling theorem for signals on arbitrary graphs,''
\newblock in {\em IEEE International Conference on Acoustics, Speech and Signal
  Processing (ICASSP)}, 2014, pp. 3864--3868.

\bibitem{chen2015discrete}
Siheng Chen, Rohan Varma, Aliaksei Sandryhaila, and Jelena Kova{\v{c}}evi{\'c},
\newblock ``Discrete signal processing on graphs: Sampling theory,''
\newblock {\em {IEEE} Trans. Signal Process.}, vol. 63, no. 24, pp. 6510--6523,
  2015.

\bibitem{belkin2004regularization}
Mikhail Belkin, Irina Matveeva, and Partha Niyogi,
\newblock ``Regularization and semi-supervised learning on large graphs,''
\newblock in {\em International Conference on Computational Learning Theory}.
  Springer, 2004, pp. 624--638.

\bibitem{belkin2006manifold}
Mikhail Belkin, Partha Niyogi, and Vikas Sindhwani,
\newblock ``Manifold regularization: A geometric framework for learning from
  labeled and unlabeled examples,''
\newblock {\em Journal of machine learning research}, vol. 7, no. Nov, pp.
  2399--2434, 2006.

\bibitem{chamon2016near}
Luiz~FO Chamon and Alejandro Ribeiro,
\newblock ``Near-optimality of greedy set selection in the sampling of graph
  signals,''
\newblock in {\em IEEE Global Conference on Signal and Information Processing
  (GlobalSIP)}, 2016, pp. 1265--1269.

\bibitem{chamon2017greedy}
Luiz~FO Chamon and Alejandro Ribeiro,
\newblock ``Greedy sampling of graph signals,''
\newblock {\em arXiv preprint arXiv:1704.01223}, 2017.

\bibitem{berman1994nonnegative}
Abraham Berman and Robert~J Plemmons,
\newblock {\em Nonnegative matrices in the mathematical sciences},
\newblock SIAM, 1994.

\bibitem{Fujishige90}
S.~Fujishige,
\newblock {\em Submodular Functions and Optimization},
\newblock Annals of Discrete Math., North Holland, 1990.

\bibitem{Fortunato10}
Santo Fortunato,
\newblock ``Community detection in graphs,''
\newblock {\em Physics Reports}, vol. 486, no. 3-5, pp. 75--174, 2010.

\bibitem{CPY14ICASSP}
Pin-Yu Chen and Alfred~O. Hero,
\newblock ``Local {Fiedler} vector centrality for detection of deep and
  overlapping communities in networks,''
\newblock in {\em IEEE International Conference on Acoustics, Speech and Signal
  Processing (ICASSP)}, 2014, pp. 1120--1124.

\bibitem{CPY14deep}
P.-Y. Chen and A.O. Hero,
\newblock ``Deep community detection,''
\newblock {\em {IEEE} Trans. Signal Process.}, vol. 63, no. 21, pp. 5706--5719,
  Nov. 2015.

\bibitem{Zachary77Karate}
Wayne~W. Zachary,
\newblock ``An information flow model for conflict and fission in small
  groups,''
\newblock {\em Journal of Anthropological Research}, vol. 33, no. 4, pp.
  452--473, 1977.

\bibitem{Lusseau03Dolphin}
David Lusseau, Karsten Schneider, OliverJ. Boisseau, Patti Haase, Elisabeth
  Slooten, and SteveM. Dawson,
\newblock ``The bottlenose dolphin community of doubtful sound features a large
  proportion of long-lasting associations,''
\newblock {\em Behavioral Ecology and Sociobiology}, vol. 54, no. 4, pp.
  396--405, 2003.

\bibitem{anis2016efficient}
Aamir Anis, Akshay Gadde, and Antonio Ortega,
\newblock ``Efficient sampling set selection for bandlimited graph signals
  using graph spectral proxies,''
\newblock {\em {IEEE} Trans. Signal Process.}, vol. 64, no. 14, pp. 3775--3789,
  2016.

\bibitem{SihengChen15signalrecovery}
Siheng Chen, A.~Sandryhaila, J.M.F. Moura, and J.~Kovacevic,
\newblock ``Signal recovery on graphs: Variation minimization,''
\newblock {\em {IEEE} Trans. Signal Process.}, vol. 63, no. 17, pp. 4609--4624,
  Sept. 2015.

\end{thebibliography}
